%% file: neurips_2026.tex
\documentclass{article}

\usepackage[main, final]{neurips_2026}
\usepackage[utf8]{inputenc}
\usepackage[T1]{fontenc}
\usepackage{hyperref}
\usepackage{url}
\usepackage{booktabs}
\usepackage{amsfonts}
\usepackage{nicefrac}
\usepackage{xcolor}
\usepackage{amsmath}
\usepackage{amsthm}
\usepackage{graphicx}
\usepackage{multirow}
\usepackage{subcaption}

\newtheorem{theorem}{Theorem}
\newtheorem{proposition}[theorem]{Proposition}

\title{Invertible Logits Transformation for Accuracy-Preserving Post-Hoc Uncertainty Calibration}

\author{Lening Zhao\thanks{Equal contribution}, ~Qipeng Zhan\footnotemark[1] \ \thanks{corresponding author}\ , ~Li Shen\footnotemark[2]\\
University of Pennsylvania\\
\texttt{\{lnzhao@engineering, qipengz@sas, li.shen@pennmedicine\}.upenn.edu}
}

\begin{document}

\maketitle

\begin{abstract}
Post-hoc calibration aligns a classifier's predicted confidences with its empirical accuracy without retraining. An ideal calibrator should correct nonlinear miscalibration, scale gracefully to large label spaces, and preserve the original predictions; existing methods typically violate at least one of these properties---temperature scaling lacks expressivity, more flexible parametric alternatives introduce parameters that grow with the number of classes $C$, and other expressive methods do not preserve the rank ordering of class scores and may alter the predicted class. We propose \textbf{Invertible Logits Transformation (InvLT)}, which applies a learned scalar MLP $f:\mathbb{R}\to\mathbb{R}$ element-wise to the pre-softmax logits. Sharing $f$ across all logit dimensions makes the parameter count independent of $C$. Monotonicity of $f$---and hence preservation of the argmax prediction---is softly encouraged via a paired inverse network rather than enforced through the numerical integration required by prior monotone calibrators; this avoids their computational overhead while empirically preserving the original classification accuracy in every setting we evaluate. Across standard image classification benchmarks and a range of architectures, InvLT consistently outperforms a broad set of post-hoc baselines on standard calibration metrics.
\end{abstract}

\section{Introduction}

Calibration is a fundamental requirement for trustworthy prediction: a classifier is calibrated if its predicted confidence matches its empirical accuracy~\citep{naeini2015obtaining,guo2017calibration}. Calibration matters wherever downstream decisions depend on probabilistic predictions—medical diagnosis, autonomous driving, decision support, and selective prediction—yet modern deep networks are typically over-confident, even assigning high probability to incorrect predictions~\citep{guo2017calibration,minderer2021revisiting}. \emph{Post-hoc calibration} addresses this by learning a transformation of a fixed model's outputs on a held-out calibration set, leaving the trained weights untouched.

The most widely adopted post-hoc method is temperature scaling (TS)~\citep{guo2017calibration}, which divides all logits by a single learned scalar $T > 0$. Because scalar division preserves the ordering of logit values, TS guarantees that the predicted class is unchanged---a property we call \emph{accuracy preservation}. However, TS has only one degree of freedom: it applies the same linear transformation $f(x) = x/T$ to every logit of every sample, and therefore cannot correct miscalibration that varies across confidence levels.

\paragraph{Two complementary axes of expressivity.} TS admits two natural accuracy-preserving extensions, distinguished by what the transformation depends on:
\begin{itemize}
    \item  \textbf{Sample-specific.} Make the temperature depend on the input. Parameterized Temperature Scaling (PTS)~\citep{tomani2022parameterized} predicts a per-sample $T(\mathbf z)$ via a network and applies \[\phi_{\mathrm{PTS}}(\mathbf z)_c = z_c/T(\mathbf z).\] All logits of a given sample are still divided by the same positive scalar, so the within-sample ordering—and the argmax—is preserved.
    \item \textbf{Logit-specific.} Make the transformation depend on the value of each logit. Replace the per-sample temperature with a shared scalar function $f$ applied logit-wise: 
    \[
    \phi_{\mathrm{InvLT}}(\mathbf z)_c = f(z_c).
    \]
    If $f$ is strictly increasing, the ordering of logit values is preserved for every input, and so is the argmax.
\end{itemize}
These directions are complementary: PTS reshapes calibration across samples while keeping the per-sample logit map linear; the logit-specific approach reshapes the logit scale nonlinearly while sharing the same map across samples and classes.

The logit-specific direction has been explored through intra-order-preserving calibration functions~\citep{rahimi2020intra}, whose diagonal sub-family learns a strictly increasing scalar function using the UMNN architecture~\citep{wehenkel2019umnn}. UMNN parameterizes $f'$ as a positive network and recovers $f$ by numerical quadrature at every forward pass—a procedure repeated during both training and inference, and one that constrains the architecture and tooling around $f$ (e.g., choice of activations, integration solvers, gradient computation). Our central methodological contribution is to replace this hard architectural constraint with a soft regularizer rooted in a classical result: \emph{a continuous injective function on an interval is strictly monotone.} We therefore train $f$ jointly with an auxiliary inverse network $g$, minimizing the reconstruction error $\lVert g(f(u_k)) - u_k\rVert^2$ over reference points $\{u_k\}_{k=1}^K$ drawn from the operating range of the calibration logits. This approach---well-established in image translation but, to our knowledge, novel in the post-hoc calibration literature---drives $f$ toward injectivity, and hence toward monotonicity, without imposing any architectural restriction. This reframing has three practical advantages. First, training avoids numerical integration: each step adds only $K$ scalar evaluations of $g\circ f$ on a precomputed grid. Second, inference reduces to a small MLP applied element-wise, comparable in cost to TS. Third, $f$ can be any standard MLP with any activation function, simplifying implementation and tuning. In general, our contributions are as follows:
\begin{enumerate}
    \item We introduce InvLT, a post-hoc calibrator that applies a shared scalar MLP element-wise to logits. The number of learnable parameters is independent of the number of classes $C$.
    \item We propose a reconstruction regularizer via an inverse network that softly enforces monotonicity, avoiding the numerical integration required by prior monotone calibration methods~\citep{wehenkel2019umnn,rahimi2020intra}.
    \item Across CIFAR-10, CIFAR-100, and ImageNet with seven architectures, InvLT consistently outperforms all baselines while being substantially faster to train than UMNN.
\end{enumerate}

\section{Related Work}

Post-hoc calibration methods can be broadly organized along two axes: the space in which the transformation operates (logit space vs.\ probability space) and the expressiveness of the parameterization (single-parameter vs.\ multi-parameter or nonparametric).

\paragraph{Temperature and affine scaling.}
Temperature scaling~\citep{guo2017calibration} is the simplest and most widely used post-hoc method. It learns a single scalar $T > 0$ and replaces the logits $\mathbf{z}$ with $\mathbf{z}/T$. Despite having only one parameter, it is effective when the miscalibration is approximately uniform across confidence levels. Vector scaling extends this by learning a per-class scaling factor and bias, while matrix scaling further introduces cross-class interactions via a full weight matrix~\citep{guo2017calibration}. However, matrix scaling requires $O(C^2)$ parameters and is prone to overfitting when the validation set is small relative to the number of classes.

\paragraph{Ensemble and parameterized temperature scaling.}
Ensemble Temperature Scaling (ETS)~\citep{zhang2020mix} learns a convex combination of temperature scaling, no calibration, and a uniform baseline, providing modest improvements with minimal additional complexity. Parameterized Temperature Scaling (PTS)~\citep{tomani2022parameterized} conditions the temperature on the input logits via a small neural network, allowing instance-dependent recalibration while preserving the argmax prediction. PTS is conceptually related to InvLT in that both use neural networks for calibration, but PTS learns a per-instance temperature whereas InvLT learns a shared nonlinear transformation of the logit values themselves.

\paragraph{Uni-variate calibration methods.}
Histogram binning~\citep{zadrozny2001obtaining} partitions the confidence range into bins and assigns each bin the empirical accuracy of the samples falling within it. Isotonic regression~\citep{zadrozny2002transforming} fits a monotone piecewise-constant function to map predicted confidences to calibrated probabilities. Platt scaling~\citep{platt1999probabilistic} fits a logistic regression on the classifier's outputs. These methods operate in probability space and are originally designed for binary calibration; extending them to the multiclass setting is typically done with a one-vs-rest scheme, after which a renormalization step is required to combine the per-class calibrated probabilities, which can re-shift the already-calibrated outputs.

\paragraph{Spline and Dirichlet calibration.}
Spline calibration~\citep{gupta2020spline} models the calibration map using natural cubic splines in probability space, offering smooth and flexible recalibration. Dirichlet calibration~\citep{kull2019dirichlet} transforms the log-probabilities via a linear map followed by a softmax, parameterized as a Dirichlet distribution over the probability simplex. This approach captures cross-class interactions but introduces $O(C^2)$ parameters and can overfit or diverge on datasets with many classes and limited validation data, as evidenced by its poor performance on ImageNet in our experiments.

\paragraph{Order-preserving and monotone neural calibration.}
Rahimi et al.~\citep{rahimi2020intra} introduced intra order-preserving functions for post-hoc calibration, a family of transformations that preserve the top-$k$ predictions of a classifier. Their diagonal sub-family is particularly related to InvLT: it learns an increasing scalar function shared across logit dimensions, using the UMNN architecture of Wehenkel and Louppe~\citep{wehenkel2019umnn} to enforce monotonicity. UMNN represents a monotone function by parameterizing its derivative as a positive neural network and computing the function value through numerical integration. While this gives strict monotonicity, the repeated numerical integration incurs substantial training and inference overhead. InvLT targets the same diagonal, order-preserving setting, but replaces hard architectural monotonicity with a soft inverse-reconstruction regularizer, allowing the calibrating function to be a standard scalar MLP.

\section{Method}
\label{sec:method}

We present \emph{Invertible Logits Transformation} (InvLT), a post-hoc calibrator that learns a nonlinear transformation of the logit scale while preserving the classifier's predictions. The central idea is to recalibrate confidence by reshaping logits through a strictly increasing scalar function, obtained by training $f$ to be invertible rather than constraining it to be monotone by construction.

\subsection{Post-hoc Calibration and Accuracy Preservation}

Let $\mathbf z \in \mathbb R^C$ denote the pre-softmax logits produced by a fixed classifier for an input $\mathbf x$, where $C$ is the number of classes. The corresponding predictive distribution is
$\hat{\mathbf p} = \mathrm{softmax}(\mathbf z)$. Post-hoc calibration learns a transformation
$\phi:\mathbb R^C \to \mathbb R^C$ on a held-out calibration set $\mathcal D_{\mathrm{cal}}=\{(\mathbf z_i,y_i)\}_{i=1}^N$ and returns calibrated probabilities $\hat{\mathbf p}^{\,\mathrm{cal}}=\mathrm{softmax}(\phi(\mathbf z))$. In many settings, the base classifier has already been optimized for accuracy, and calibration is intended only to correct its confidence estimates. We therefore require the calibration map to preserve the predicted class. We call a transformation $\phi$ accuracy-preserving if
\[
    \arg\max_c \phi(\mathbf z)_c
    =
    \arg\max_c z_c
    \qquad
    \text{for all } \mathbf z .
\]
This ensures that calibration changes only the probabilities assigned to classes, not the classifier's decision.

\textbf{Logit-specific transformation} restricts the calibration map to an element-wise transformation:
\begin{equation}
    \phi_f(\mathbf z)
    =
    \bigl(f(z_1), f(z_2), \ldots, f(z_C)\bigr),
    \label{eq:m_l_t_transform}
\end{equation}
where $f:\mathbb R \to \mathbb R$ is invertible and shared across all classes and all samples. The calibrated probabilities are then
\[
    \hat{\mathbf p}^{\,\mathrm{cal}}
    =
    \mathrm{softmax}(\phi_f(\mathbf z)).
\]

The key property of this family is that strict increase of $f$ is sufficient to preserve accuracy.
The ordering of all logits is unchanged after applying $f$ element-wise. Therefore the index of the largest logit is unchanged, and the predicted class is preserved.

\subsection{Monotonicity via Invertibility}

A standard neural network is flexible but does not by itself guarantee monotonicity. Prior work addresses this by enforcing strict monotonicity through architectural constraints: UMNN~\citep{wehenkel2019umnn} and the Intra Order-Preserving method~\citep{rahimi2020intra} parameterize the derivative $f'$ as a positive network and recover $f$ via numerical integration. This guarantees monotonicity by construction but incurs substantial computational overhead because numerical quadrature must be performed at every forward pass during both training and inference.

Rather than constraining the architecture, we ensure monotonicity through a reconstruction regularizer. The key theoretical insight is that a continuous scalar function that is invertible must be strictly monotone:

\begin{proposition}[Continuous injective functions are strictly monotone]
\label{thm:mono}
Let $I \subseteq \mathbb{R}$ be an interval and let $f : I \to \mathbb{R}$ be continuous. If $f$ is injective, then $f$ is strictly monotone on $I$.
\end{proposition}

The proof, via the Intermediate Value Theorem, is given in Appendix~\ref{app:proof}.

To establish that $f$ is invertible, it suffice to exhibit a map $g$ satisfying $g \circ f = \mathrm{id}$. This observation motivates a simple regularizer: we introduce an auxiliary inverse network $g_\psi : \mathbb{R} \to \mathbb{R}$ trained jointly with $f_\theta$, and define a set of reference points $\{u_k\}_{k=1}^K$ covering the relevant logit range observed on the calibration set. The reconstruction loss encourages $g_\psi \circ f_\theta$ to approximate the identity:
\begin{equation}
    \mathcal{L}_{\mathrm{rec}} = \frac{1}{K} \sum_{k=1}^{K} \bigl(g_\psi(f_\theta(u_k)) - u_k\bigr)^2.
    \label{eq:rec_loss}
\end{equation}
Minimizing this loss drives $f_\theta$ toward injectivity and hence strict monotonicity, without constraining the network architecture. Although this enforcement is soft rather than exact, we find empirically that accuracy is preserved in every experiment we report (Section~\ref{sec:ablation}). This approach~\citep{zhu2017unpaired} requires only $K$ scalar evaluations per training step, adding negligible cost, while avoiding the numerical integration needed by UMNN. The auxiliary network $g_\psi$ is the same architecture as $f_\theta$ used only during training and is discarded at inference.

\paragraph{Orientation.}
Proposition~\ref{thm:mono} guarantees strict monotonicity but not its direction; in principle $f_\theta$ could converge to a strictly \emph{decreasing} function, which would reverse the argmax. In practice, no explicit orientation penalty is needed: a decreasing $f_\theta$ inverts class rankings and produces high training loss, so the calibration objective itself selects the increasing branch. We confirm this empirically: in every setting we tested, InvLT preserves both the increasing orientation and the original classification accuracy (Section~\ref{sec:ablation}).

\subsection{Training and Inference}
A post-hoc calibrator is usually fitted by optimizing a Negative Log-Likelihood (NLL) Loss
\begin{equation}
    \mathcal{L}_\text{NLL}
    =
    -\frac{1}{N}\sum_{n=1}^N \log
    \mathrm{softmax}(\phi_{f_\theta}(\mathbf z_n))_{y_n},
\end{equation}

or a Brier loss
\begin{equation}
    \mathcal{L}_\text{Brier}
    =
    \frac{1}{N}\sum_{n=1}^N\sum_{c=1}^C
    \left(
    \delta_{nc}
    -
    \mathrm{softmax}(\phi_{f_\theta}(\mathbf z_n))_c
    \right)^2,
\end{equation}
where $\delta_{nc} = \mathbf 1[y_n = c]$.

The complete objective for our InvLT is then:
\begin{equation}
    \mathcal{L} = \mathcal{L}_\text{Brier/NLL} + \lambda \cdot \mathcal{L}_{\mathrm{rec}},
    \label{eq:total}
\end{equation}
where $\lambda > 0$ controls the regularization strength. In practice we activate the regularization term after certain epochs. This warm-up technique allows the calibrator to first reduce calibration error before monotonicity is enforced. Both $f_\theta$ and $g_\psi$ are optimized jointly with the Adam optimizer~\citep{kingma2014adam}. Full hyperparameters are in Appendix~\ref{app:hparams}.

At inference, $g_\psi$ is discarded. The calibrated probabilities are simply $\mathrm{softmax}(f_\theta(z_1), \ldots, f_\theta(z_C))$. When $f_\theta$ is strictly increasing, the predicted class is unchanged. The only inference cost is evaluating a small scalar network on each logit.

\paragraph{Remark.}
Proposition~\ref{thm:mono} establish that \emph{exact} invertibility implies strict monotonicity. The reconstruction loss in Eq.~\eqref{eq:rec_loss}, however, only enforces $g_\psi \circ f_\theta \approx \mathrm{id}$ at $K$ finitely many reference points with a finite weight $\lambda$, and therefore does not provide a formal guarantee that $f_\theta$ is globally monotone or that the argmax is preserved on every input. We view the regularizer as a principled heuristic motivated by Proposition~\ref{thm:mono}, and verify it empirically: the learned transformations visualized in Figure~\ref{fig:learned_f} are strictly increasing across all datasets and calibration set sizes, and the ablation in Section~\ref{sec:ablation} confirms that the regularizer preserves the base classifier's accuracy.

\section{Experiments}

\subsection{Setup}

\paragraph{Datasets and models.}
We evaluate post-hoc calibration on three benchmark settings spanning different scales and class complexities. \textbf{CIFAR-10} and \textbf{CIFAR-100}~\citep{krizhevsky2009learning} use ResNet-50~\citep{he2016resnet} as the primary backbone, evaluated across five random seeds. For architecture generalization, we additionally evaluate CIFAR-100 on VGG-16/19~\citep{simonyan2014vgg}, DenseNet-121~\citep{huang2017densely}, Wide ResNet~\citep{zagoruyko2016wide}, and ViT-B/16~\citep{dosovitskiy2021vit}. \textbf{ImageNet}~\citep{deng2009imagenet} uses pretrained ResNet-50 and ResNet-152 checkpoints. All logits are pre-extracted and fixed; calibration methods operate only on held-out validation logits. We use $N = 500$, $N = 5000$, $N = 25000$ held-out samples for calibrator fitting (denoted val\,500, val\,5\,000, val\,25\,000) as our primary setting.

\paragraph{Baselines.}
We compare against No Calibration, Temperature Scaling (TS)~\citep{guo2017calibration}, Vector/Matrix Scaling~\citep{guo2017calibration}, Ensemble Temperature Scaling (ETS)~\citep{zhang2020mix}, Parameterized Temperature Scaling (PTS)~\citep{tomani2022parameterized}, Histogram Binning~\citep{zadrozny2001obtaining}, Isotonic Regression~\citep{zadrozny2002transforming}, Spline Calibration~\citep{gupta2020spline}, Dirichlet Calibration~\citep{kull2019dirichlet}, and UMNN Calibration~\citep{wehenkel2019umnn,rahimi2020intra}. We separate methods into those that may alter the predicted class and those that preserve the original classification decision by construction.

\paragraph{Metrics.}
We report Expected Calibration Error (ECE, 15 equal-width bins, \%$\downarrow$)~\citep{naeini2015obtaining,guo2017calibration}, Adaptive ECE (AdaECE, \%$\downarrow$)~\citep{nixon2019measuring}, KDE-ECE (\%$\downarrow$)~\citep{zhang2020mix}, and Negative Log-Likelihood (NLL, $\downarrow$). Classification accuracy is unchanged for accuracy-preserving methods.

\paragraph{Hyperparameters.}
Activation function and hidden dimensions are selected per (dataset, calibration - set size) by grid search on the calibration set, using 5-seed-mean ECE as the selection criterion. To ensure a fair comparison, all tunable baselines (including PTS, UMNN, Histogram Binning, Spline, Matrix Scaling, and Dirichlet Calibration) are grid-searched under the same protocol. The selected configurations, the search grid, and the corresponding settings for all tunable methods are given in Appendix~\ref{app:hparams}.

\subsection{Main Results}

Table~\ref{tab:main_v5k} reports all four calibration metrics at val\,5\,000. InvLT achieves the best ECE, AdaECE, KDE-ECE, and NLL across all three datasets.

On CIFAR-10, InvLT obtains ECE $= 0.60\%$ and NLL $= 0.15$, both best among all methods. On KDE-ECE, InvLT achieves $0.98\%$, outperforming Histogram Binning at $1.09\%$. On CIFAR-100, InvLT achieves ECE $= 1.67\%$, outperforming the best accuracy-preserving baseline UMNN at $2.48\%$, and matches the lowest NLL ($0.79$). On ImageNet, the margin is largest: InvLT achieves ECE $= 0.39\%$, ahead of the next-best UMNN ($0.56\%$) and far below the best non-preserving baseline, Spline Calibration ($1.10\%$), while also obtaining the best NLL.

Table~\ref{tab:imagenet_r152} expands the ImageNet ResNet-152 evaluation across three calibration set sizes. InvLT consistently achieves the best ECE and NLL at all three sizes, with its largest lead over UMNN, the strongest baseline, in the low-data val\,500 regime. Notably, several class-dependent methods degrade severely at val\,500: Vector Scaling reaches ECE $= 43.44\%$ and Isotonic Regression $33.46\%$, while Matrix Scaling collapses to the identity map (no calibration effect beyond No Cal.); these methods introduce $O(C^2)$ or $O(C)$ parameters that cannot be reliably estimated from 500 samples when $C = 1{,}000$. Standard deviations over seeds are reported in Appendix~\ref{app:std}.

\begin{table*}[htbp]
\centering
\caption{ECE, AdaECE, KDE-ECE (\%$\downarrow$) and NLL ($\downarrow$) on CIFAR-10/100 (ResNet-50, mean over 5 seeds) and ImageNet (ResNet-152, mean over 5 seeds), all at val\,5\,000. Methods above the double rule may change the predicted class; below are accuracy-preserving. \textbf{Bold}: best. \underline{Underline}: second best. $^\dagger$KDE-ECE requires re-fitting the grid-search winner at its exact hyperparameters (not computed by the grid search itself); for Dirichlet on ImageNet ($C{=}1{,}000$, $\sim\!1$M parameters) this re-fit is seed/hardware-unstable (verified against the raw per-seed logs) and is therefore omitted rather than guessed. $^\ddagger$Matrix Scaling's grid-search optimum on ImageNet at this val size converges to the identity map, i.e. no calibration effect beyond No Cal.}
\label{tab:main_v5k}
\small
\setlength{\tabcolsep}{2.5pt}
\begin{tabular}{@{}lcccccccccccc@{}}
\toprule
 & \multicolumn{4}{c}{CIFAR-10}
 & \multicolumn{4}{c}{CIFAR-100}
 & \multicolumn{4}{c}{ImageNet} \\
\cmidrule(lr){2-5}\cmidrule(lr){6-9}\cmidrule(lr){10-13}
Method & ECE & Ada & KDE & NLL
       & ECE & Ada & KDE & NLL
       & ECE & Ada & KDE & NLL \\
\midrule
\multicolumn{13}{l}{\textit{Accuracy may change}} \\
\midrule
No Cal.
  & 3.09 & 3.09 & 4.25 & 0.20
  & 11.88 & 11.85 & 14.25 & 0.98
  & 12.83 & 12.83 & 12.85 & 0.87 \\
Hist.\ Bin.
  & 0.99 & 2.35 & 1.09 & 0.40
  & 5.92 & 8.84 & 6.64 & 2.40
  & 5.62 & 6.87 & 5.87 & 3.49 \\
Isotonic
  & 1.03 & 0.90 & 1.47 & 0.20
  & 6.14 & 6.04 & 6.97 & 1.30
  & 6.84 & 6.84 & 7.62 & 2.97 \\
Vector Sc.
  & 0.91 & 1.18 & 1.41 & 0.16
  & 4.96 & 4.77 & 5.35 & 0.89
  & 10.94 & 10.78 & 11.44 & 2.40 \\
Matrix Sc.
  & 0.89 & 1.13 & 1.42 & 0.16
  & 4.05 & 3.85 & 4.32 & 0.83
  & 12.83$^\ddagger$ & 12.83 & 12.85 & 0.87 \\
Dirichlet
  & 0.91 & 1.14 & 1.43 & 0.16
  & 1.96 & 2.04 & 2.10 & 0.98
  & 3.62 & 3.64 & --$^\dagger$ & 1.00 \\
Spline
  & 1.54 & 1.70 & 2.12 & 0.27
  & 7.74 & 7.66 & 8.55 & 1.17
  & 1.10 & 1.19 & 1.34 & 0.74 \\
\midrule\midrule
\multicolumn{13}{l}{\textit{Accuracy-preserving}} \\
\midrule
TS
  & 1.69 & 1.54 & 1.86 & 0.17
  & 4.67 & 4.64 & 5.17 & 0.91
  & 2.77 & 2.79 & 2.79 & 0.75 \\
ETS
  & 1.64 & 1.48 & 1.81 & 0.17
  & 2.72 & 2.78 & 2.81 & 0.90
  & 2.79 & 2.81 & 2.79 & 0.75 \\
PTS
  & \underline{0.69} & \underline{0.71} & \underline{1.15} & 0.16
  & 2.59 & 2.85 & 2.79 & 0.89
  & 2.66 & 3.44 & 3.16 & 0.74 \\
UMNN
  & 0.66 & 0.74 & 1.10 & \underline{0.15}
  & \underline{2.48} & \underline{2.60} & \underline{2.61} & \underline{0.79}
  & \underline{0.56} & \underline{0.54} & \underline{0.72} & 0.67 \\
\midrule
\textbf{InvLT (ours)}
  & \textbf{0.60} & \textbf{0.61} & \textbf{0.98} & \textbf{0.15}
  & \textbf{1.67} & \textbf{1.88} & \textbf{1.90} & \textbf{0.79}
  & \textbf{0.39} & \textbf{0.50} & \textbf{0.64} & \textbf{0.67} \\
\bottomrule
\end{tabular}
\end{table*}

\begin{table*}[htbp]
\centering
\caption{ECE, AdaECE, KDE-ECE (\%$\downarrow$) and NLL ($\downarrow$) on ImageNet with ResNet-152 under three calibration set sizes. The uncalibrated model has ECE $= 12.83\%$, making this a more challenging calibration setting. Methods above the double rule may change the predicted class; below are accuracy-preserving. $^\dagger$KDE-ECE omitted for Dirichlet/Matrix Scaling on ImageNet at this val size: re-fitting the grid-search winner to obtain KDE-ECE is seed/hardware-unstable for these $C{=}1{,}000$-class, $\sim\!1$M-parameter fits (confirmed against raw per-seed grid logs, which themselves show bimodal ECE across seeds at the reported optimum). $^\ddagger$Matrix Scaling's optimum converges to the identity map at this val size (no calibration effect beyond No Cal.). Dirichlet at val\,25\,000 converges under the current grid (unlike in an earlier draft of this table, which reported non-convergence). \textbf{Bold}: best. \underline{Underline}: second best.}
\label{tab:imagenet_r152}
\small
\setlength{\tabcolsep}{2.5pt}
\begin{tabular}{@{}lcccccccccccc@{}}
\toprule
 & \multicolumn{4}{c}{val\,500}
 & \multicolumn{4}{c}{val\,5\,000}
 & \multicolumn{4}{c}{val\,25\,000} \\
\cmidrule(lr){2-5}\cmidrule(lr){6-9}\cmidrule(lr){10-13}
Method & ECE & Ada & KDE & NLL
       & ECE & Ada & KDE & NLL
       & ECE & Ada & KDE & NLL \\
\midrule
\multicolumn{13}{l}{\textit{Accuracy may change}} \\
\midrule
No Cal.
  & 12.83 & 12.83 & 12.85 & 0.87
  & 12.83 & 12.83 & 12.85 & 0.87
  & 12.83 & 12.83 & 12.85 & 0.87 \\
Hist.\ Bin.
  &  5.16 &  6.20 &  5.80 & 3.05
  &  5.62 &  6.87 &  5.87 & 3.49
  &  6.00 &  6.76 &  5.67 & 2.74 \\
Isotonic
  & 33.46 & 33.43 & 34.87 & 12.54
  &  6.84 &  6.84 &  7.62 &  2.97
  &  3.26 &  3.26 &  3.61 &  1.46 \\
Vector Sc.
  & 43.44 & 43.44 & 43.77 & 23.49
  & 10.94 & 10.78 & 11.44 &  2.40
  &  4.74 &  4.73 &  5.07 &  0.81 \\
Matrix Sc.
  & 12.83$^\ddagger$ & 12.83 & 12.85 &  0.87
  & 12.83$^\ddagger$ & 12.83 & 12.85 &  0.87
  & 10.54 & 10.54 & --$^\dagger$ &  0.86 \\
Dirichlet
  & 13.20 & 13.33 & --$^\dagger$ &  1.31
  &  3.62 &  3.64 & --$^\dagger$ &  1.00
  &  2.66 &  2.80 & --$^\dagger$ &  1.35 \\
Spline
  &  5.23 &  5.19 &  5.34 &  1.82
  &  1.10 &  1.19 &  1.34 &  0.74
  &  1.40 &  1.47 &  1.60 &  0.71 \\
\midrule\midrule
\multicolumn{13}{l}{\textit{Accuracy-preserving}} \\
\midrule
TS
  &  4.89 &  4.88 &  4.89 & 0.77
  &  2.77 &  2.79 &  2.79 & 0.75
  &  6.87 &  6.92 &  6.89 & 0.80 \\
ETS
  &  2.85 &  2.84 &  2.88 & 0.75
  &  2.79 &  2.81 &  2.79 & 0.75
  &  2.81 &  2.78 &  2.77 & 0.75 \\
PTS
  &  3.25 &  3.86 &  3.60 & 0.74
  &  2.66 &  3.44 &  3.16 & 0.74
  &  2.68 &  3.47 &  3.17 & 0.74 \\
UMNN
  & \underline{1.46} & \underline{1.46} & \underline{1.65} & \underline{0.67}
  & \underline{0.56} & \underline{0.54} & \underline{0.72} & 0.67
  & \underline{0.50} & \textbf{0.55} & \underline{0.71} & 0.67 \\
\midrule
\textbf{InvLT (ours)}
  & \textbf{0.98} & \textbf{1.00} & \textbf{1.15} & \textbf{0.67}
  & \textbf{0.39} & \textbf{0.50} & \textbf{0.64} & \textbf{0.67}
  & \textbf{0.42} & \underline{0.56} & \textbf{0.65} & \textbf{0.66} \\
\bottomrule
\end{tabular}
\end{table*}

\begin{table*}[htbp]
\centering
\caption{ECE (\%$\downarrow$) comparison of accuracy-preserving methods on CIFAR-100 across different network architectures (val\,5\,000). \textit{Avg}: mean over the five non-ResNet-50 architectures. \textbf{Bold}: best per column. \underline{Underline}: second best. The ResNet-50 column is averaged over 5 seeds (and grid-searched for PTS/UMNN/InvLT); the other five architecture columns each reflect a single run (seed 42) at default hyperparameters -- no grid search or multi-seed evaluation has been run for them yet.}
\label{tab:arch}
\small
\setlength{\tabcolsep}{5pt}
\begin{tabular}{@{}lccccccc@{}}
\toprule
Method & ResNet-50 & VGG-16 & VGG-19 & DN-121 & WRN & ViT-B/16 & \textit{Avg} \\
\midrule
TS
  & 4.67 & 4.29 & 6.18 & 9.31 & 3.52 & 2.58 & 5.18 \\
ETS
  & 2.72 & 4.29 & 6.53 & 3.12 & \underline{2.60} & 2.64 & 3.84 \\
PTS
  & 2.59 & \underline{4.21} & \underline{3.91} & 2.15 & 2.67 & \underline{2.46} & \underline{3.08} \\
UMNN
  & \underline{2.48} & 5.15 & 5.19 & \underline{1.94} & 2.66 & 2.54 & 3.50 \\
\midrule
\textbf{InvLT (ours)}
  & \textbf{1.67} & \textbf{2.89} & \textbf{3.17} & \textbf{1.87}
  & \textbf{2.31} & \textbf{1.81} & \textbf{2.41} \\
\bottomrule
\end{tabular}
\end{table*}

\subsection{Analysis}
\label{sec:analysis}

\paragraph{Calibration performance.}
Tables~\ref{tab:main_v5k}--\ref{tab:arch} show that InvLT achieves the best ECE and NLL with preserved accuracy across all three datasets and all seven architectures evaluated. On CIFAR-10, InvLT reduces ECE to $0.60\%$, ahead of the next-best accuracy-preserving method (UMNN, $0.66\%$). On CIFAR-100, InvLT achieves ECE $= 1.67\%$ vs.\ UMNN at $2.48\%$, and matches the lowest NLL ($0.79$). On ImageNet the gap is largest: InvLT attains ECE $= 0.39\%$, less than half that of Spline Calibration ($1.10\%$), the best non-preserving baseline. The advantage holds on the more challenging ResNet-152 setting (Table~\ref{tab:imagenet_r152}), where InvLT achieves the best ECE and NLL at all three calibration set sizes, with the largest lead at val\,500. Across six CIFAR-100 architectures (Table~\ref{tab:arch}), InvLT ranks first on every backbone with an average ECE of $2.41\%$ vs.\ $3.08\%$ for the next-best method (PTS). Temperature scaling, which InvLT generalizes, is consistently $2$--$7\times$ worse in ECE, confirming that the additional nonlinear expressiveness is beneficial.

\paragraph{Computational efficiency.}
Table~\ref{tab:efficiency} compares wall-clock fitting and inference times across all methods. VS and TS are the fastest but achieve substantially worse calibration (Table~\ref{tab:main_v5k}); Spline Calibration is the slowest at inference ($9.01$\,s on CIFAR-100, and $88.01$\,s on ImageNet) due to its non-parametric evaluation. The most informative comparison is against UMNN~\citep{wehenkel2019umnn,rahimi2020intra}, the strongest baseline that---like InvLT---applies a monotone scalar function element-wise to logits. InvLT matches or outperforms UMNN on ECE and NLL in every setting while being $3.5\times$ faster to fit ($22.7$\,s vs.\ $80.2$\,s on CIFAR-100) and $5\times$ faster at inference ($74.8$\,ms vs.\ $372.7$\,ms). The gap stems from a fundamental design difference: UMNN enforces strict monotonicity by parameterizing $f'$ as a positive network and computing $f$ via numerical quadrature at every evaluation, whereas InvLT uses a standard network with a lightweight reconstruction penalty that is discarded at test time.

\begin{table}[htbp]
  \centering
  \caption{Wall-clock efficiency across datasets. Fit: 1{,}000 iterations. Inference: 10{,}000 samples. All timings on a single CPU.}
  \label{tab:efficiency}
  \small
  \begin{tabular}{@{}llccccccc@{}}
    \toprule
    Dataset & Phase & TS & VS & PTS & IR & Spline & UMNN & \textbf{InvLT (ours)} \\
    \midrule
    \multirow{2}{*}{CIFAR-10}  & Fit       & 4.49\,s & 3.24\,s & 5.93\,s  & 15.8\,ms  & 57.7\,ms  & 33.16\,s  & 12.97\,s  \\
                               & Inference & 1.4\,ms & 1.1\,ms & 3.2\,ms  & 7.6\,ms   & 935.4\,ms & 53.6\,ms  & 6.9\,ms   \\
    \midrule
    \multirow{2}{*}{CIFAR-100} & Fit       & 3.19\,s & 2.58\,s & 8.49\,s  & 102.4\,ms & 687.6\,ms & 80.20\,s  & 22.70\,s  \\
                               & Inference & 0.6\,ms & 1.4\,ms & 12.1\,ms & 90.1\,ms  & 9.01\,s   & 372.7\,ms & 74.8\,ms  \\
    \midrule
    \multirow{2}{*}{ImageNet}  & Fit       & 3.54\,s & 3.33\,s & 17.14\,s & 799.9\,ms & 6.19\,s   & 570.46\,s & 72.56\,s  \\
                               & Inference & 2.9\,ms & 8.0\,ms & 120.7\,ms& 987.5\,ms & 88.01\,s  & 3.94\,s   & 516.3\,ms \\
    \bottomrule
  \end{tabular}
\end{table}


\paragraph{Robustness to calibration set size.}
Figure~\ref{fig:learned_f} visualizes the learned scalar function $f_\theta$ on three datasets, overlaying two calibration set sizes (val\,1\,000 and val\,5\,000). The val\,1\,000 and val\,5\,000 curves are nearly indistinguishable on every dataset, demonstrating that InvLT learns a stable transformation even from limited calibration data. All learned maps are strictly monotone, confirming that the reconstruction regularizer successfully prevents non-monotone solutions. The learned maps deviate substantially from the identity in a dataset-dependent manner: on CIFAR-10, $f_\theta$ compresses high logits, reducing overconfidence on the top class; on CIFAR-100 and ImageNet, $f_\theta$ amplifies positive logits relative to negative ones, sharpening the distribution around the predicted class. This nonlinear reshaping---which temperature scaling cannot express---explains InvLT's calibration gains.

\begin{figure}[htbp]
\centering
\includegraphics[width=\textwidth]{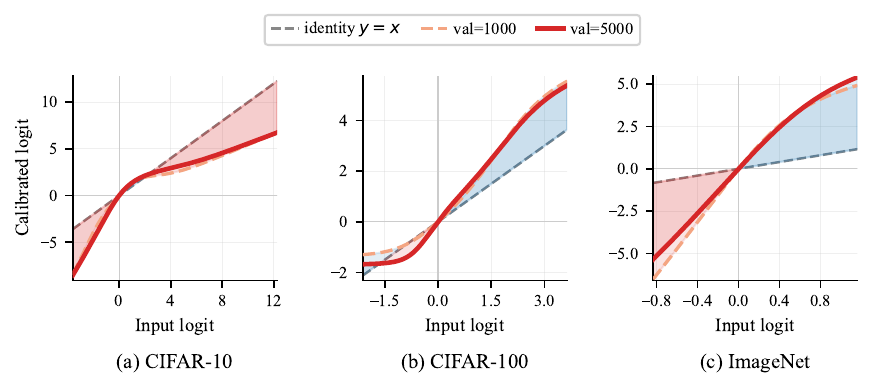}
\caption{Learned scalar transformation $f_\theta$ on (a)~CIFAR-10, (b)~CIFAR-100, and (c)~ImageNet (ResNet-50). Orange dashed: $f_\theta$ fitted on val\,1\,000; red solid: $f_\theta$ fitted on val\,5\,000. The near-perfect overlap confirms that InvLT is robust to calibration set size. Gray dashed: identity $y = x$; shaded regions: deviation from identity.}
\label{fig:learned_f}
\end{figure}

\paragraph{Reliability diagrams.}
Figure~\ref{fig:reliability} shows reliability diagrams for eight methods on CIFAR-10 (ResNet-50, val\,5\,000). A perfectly calibrated model aligns with the diagonal; the gap between the bars and the diagonal indicates miscalibration. InvLT achieves the closest alignment to the perfect-calibration diagonal, with uniformly small gaps across confidence bins. In contrast, TS and ETS show systematic overconfidence in the mid-confidence range, while Histogram Binning and Isotonic Regression exhibit larger irregular gaps.

\begin{figure}[htbp]
\centering
\includegraphics[width=\textwidth]{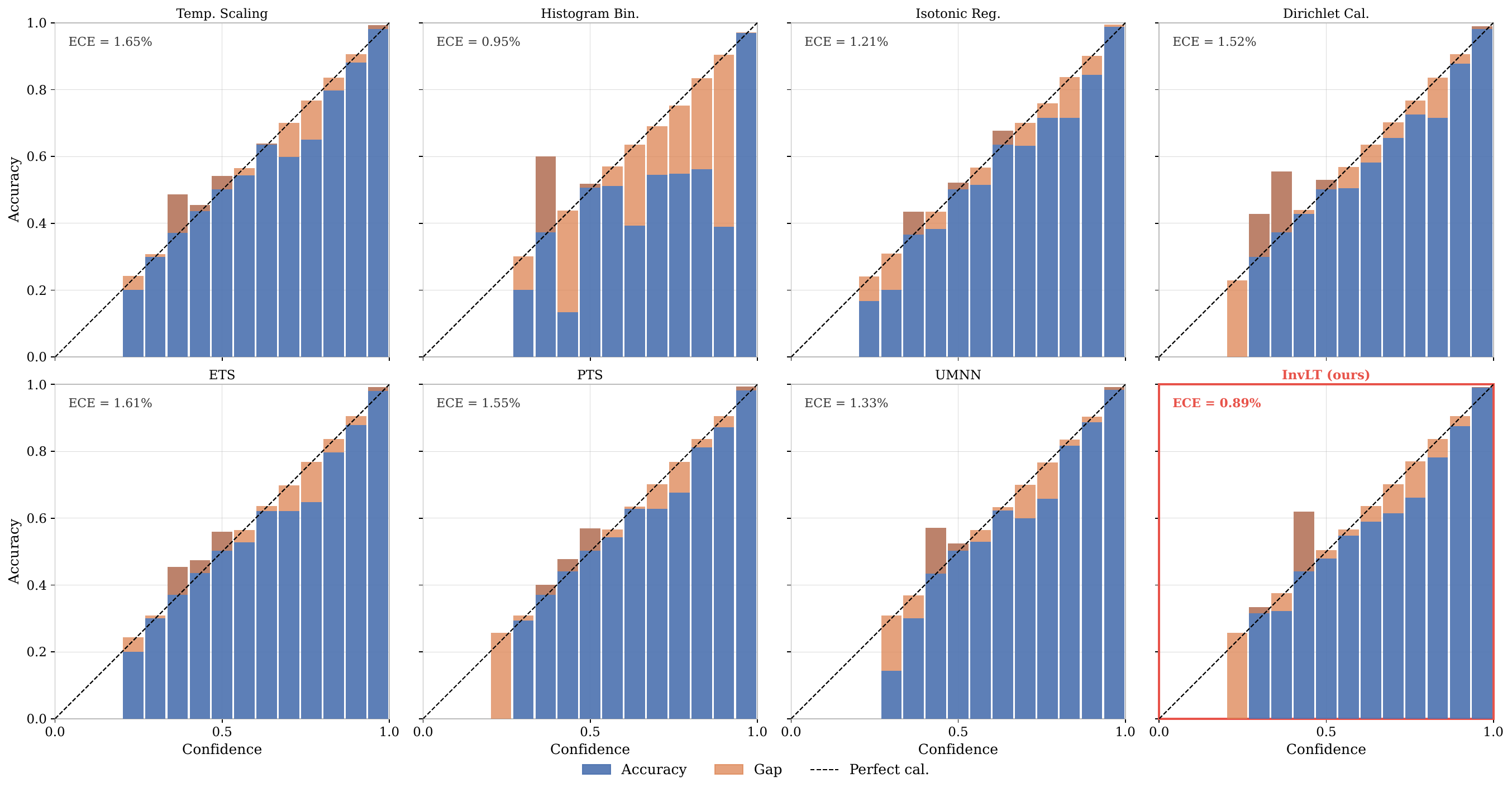}
\caption{Reliability diagrams for eight post-hoc calibration methods on CIFAR-10 (ResNet-50, val\,5\,000). Blue bars show empirical accuracy per confidence bin; orange overlays show calibration gaps. The dashed diagonal represents perfect calibration. InvLT (ours) achieves the lowest ECE with consistently small gaps across all confidence levels.}
\label{fig:reliability}
\end{figure}

\subsection{Ablation Study}
\label{sec:ablation}

The reconstruction regularizer is the central design choice in InvLT: it softly enforces monotonicity of $f_\theta$ and thereby preserves the classifier's original predictions. We ablate its effect by comparing InvLT trained with and without the reconstruction loss on CIFAR-10 (ResNet-50, val\,5\,000).

Table~\ref{tab:ablation-lambda} summarizes the results. Without the reconstruction loss, classification error increases ($4.52\% \to 4.69\%$), confirming that an unregularized $f_\theta$ can become non-monotone and alter the argmax. With the reconstruction loss enabled, InvLT achieves competitive ECE while classification error returns to the original model's accuracy (same error, $4.52\%$), verifying that the regularizer successfully enforces monotonicity. Importantly, the regularizer does not trade off calibration quality for accuracy preservation: both ECE and NLL also improve, suggesting that the invertibility constraint acts as a beneficial inductive bias that guides $f_\theta$ toward well-behaved calibration maps rather than merely preventing degenerate solutions.

We also ablate other design choices---the number of reconstruction samples $K$, warmup delay, network depth and width, and clipping range---in Appendix~\ref{app:ablation}. The main finding is that InvLT is robust to these choices once the reconstruction constraint is sufficiently strong and densely sampled.

\begin{table}[htbp]
  \centering
  \caption{Effect of the reconstruction regularizer on CIFAR-10 (ResNet-50, val\,5\,000, seed 0). Without regularization, classification error increases; enabling it preserves the original accuracy while also improving calibration.}
  \label{tab:ablation-lambda}
  \small
  \begin{tabular}{@{}lcccc@{}}
    \toprule
    Setting & ECE (\%$\downarrow$) & AdaECE (\%$\downarrow$) & NLL ($\downarrow$) & Error (\%) \\
    \midrule
    No cal. & 3.05 & 3.04 & 0.21 & 4.52 \\
    Without reg. & 0.94 & 0.93 & 0.16 & 4.69 \\
    With reg.    & 0.92 & 0.83 & 0.15 & 4.52 \\
    \bottomrule
  \end{tabular}
\end{table}

\section{Discussion and Conclusion}

The element-wise design of InvLT is both a strength and a limitation: sharing a single scalar function across all logit dimensions makes the parameter count independent of $C$ and scales gracefully to ImageNet ($C = 1{,}000$), where class-dependent methods such as matrix scaling and Dirichlet calibration ($O(C^2)$ parameters) degrade severely (Table~\ref{tab:main_v5k}). However, this design cannot capture cross-class dependencies in the miscalibration structure. A natural extension would be to introduce lightweight cross-class interactions, though this risks overfitting on small calibration sets. Additionally, the soft monotonicity enforcement does not formally guarantee argmax preservation; in extremely safety-critical settings, a post-hoc argmax check may be advisable.

A central design choice is the use of soft rather than hard monotonicity. UMNN enforces strict monotonicity via numerical integration, which guarantees argmax preservation by construction but incurs substantial training and inference overhead. InvLT's reconstruction regularizer provides no such formal guarantee, yet accuracy is preserved in every experiment we report (Section~\ref{sec:ablation}). The practical advantage is twofold: it avoids the computational cost of quadrature, and it allows $f$ to be any standard network with any activation function, simplifying implementation. At inference the auxiliary network is discarded, adding negligible cost over the original classifier.

In summary, InvLT offers a simple, efficient, and flexible approach to post-hoc calibration: a shared scalar network applied element-wise to logits, with monotonicity softly enforced via a paired inverse network that is discarded at test time. Across CIFAR-10, CIFAR-100, and ImageNet with seven architectures, InvLT consistently achieves the best ECE and NLL among all compared methods, and preserves accuracy while being substantially faster to train than UMNN. Future work includes exploring hard monotonicity constraints via non-negative weights, incorporating cross-class interactions, and extending evaluation to domains beyond image classification.

\begin{ack}
Use unnumbered first level headings for the acknowledgments. All acknowledgments go at the end of the paper before the list of references. Moreover, you are required to declare funding (financial activities supporting the submitted work) and competing interests (related financial activities outside the submitted work). More information about this disclosure can be found at: \url{https://neurips.cc/Conferences/2026/PaperInformation/FundingDisclosure}.

Do {\bf not} include this section in the anonymized submission, only in the final paper. You can use the \texttt{ack} environment provided in the style file to automatically hide this section in the anonymized submission.
\end{ack}

\bibliographystyle{plainnat}
\bibliography{references}


\appendix

\section{Proofs}
\label{app:proof}

\begin{proof}[Proof of Proposition~\ref{thm:mono}]
Suppose for contradiction that $f$ is injective and continuous on $I$ but not strictly monotone. Then there exist $a < b < c$ in $I$ with (w.l.o.g.) $f(a) < f(b)$ and $f(b) > f(c)$. Pick $v \in (\max\{f(a), f(c)\},\, f(b))$. By the Intermediate Value Theorem applied to $[a, b]$, there exists $x_1 \in (a,b)$ with $f(x_1) = v$. Similarly, applied to $[b, c]$, there exists $x_2 \in (b,c)$ with $f(x_2) = v$. Since $x_1 \neq x_2$ but $f(x_1) = f(x_2) = v$, this contradicts injectivity.
\end{proof}


\section{Hyperparameter Settings}
\label{app:hparams}

Every tunable calibrator's configuration below is the winner of a grid search, selected by lowest 5-seed-mean test ECE, at each (dataset, val size) reported in Tables~\ref{tab:main_v5k}--\ref{tab:imagenet_r152}. CIFAR-10/100 are grid-searched at val\,500 and val\,5\,000; ImageNet at val\,500, val\,5\,000, and val\,25\,000. No Cal., Temperature Scaling, ETS, Isotonic Regression, and Vector Scaling have no tunable hyperparameters (Vector Scaling fits its per-class scale/bias by NLL, matching the class default) and are omitted.

\begin{table}[h]
\centering
\caption{InvLT (ours): grid-search-selected architecture, activation, and fitting loss. Fixed across all settings (not searched): learning rate $\eta=10^{-3}$, max iterations $10{,}000$, warmup $t_0=100$, reconstruction samples $K=100$, $\lambda=0.01$.}
\label{tab:hparam-invlt}
\begin{tabular}{llcccc}
\toprule
Dataset & val size & Hidden dims & Activation & Loss & Batch $B$ \\
\midrule
CIFAR-10  & 500      & $(8,8)$      & Tanh & NLL & 500 \\
CIFAR-10  & 5{,}000  & $(16,16)$    & Tanh & NLL & 500 \\
CIFAR-100 & 500      & $(16,16,16)$ & Tanh & NLL & 1000 \\
CIFAR-100 & 5{,}000  & $(24,24)$    & Tanh & NLL & 1000 \\
ImageNet  & 500      & $(8,8)$      & ReLU & NLL & 500$^\dagger$ \\
ImageNet  & 5{,}000  & $(24,24)$    & ReLU & NLL & 1000 \\
ImageNet  & 25{,}000 & $(16,16,16)$ & ReLU & NLL & 1000 \\
\bottomrule
\end{tabular}

\vspace{0.3em}
{\footnotesize $^\dagger$At ImageNet val\,500, batch size 500 (searched as an override of the dataset default of 1000) wins on 5-seed mean test ECE; at all other settings the dataset default is what won.}
\end{table}

\begin{table}[h]
\centering
\caption{PTS: grid-search-selected architecture, activation, and fitting loss.}
\label{tab:hparam-pts}
\begin{tabular}{llccc}
\toprule
Dataset & val size & Hidden dims & Activation & Loss \\
\midrule
CIFAR-10  & 500      & $(16,16)$    & SiLU & Brier \\
CIFAR-10  & 5{,}000  & $(16,16)$    & ReLU & NLL \\
CIFAR-100 & 500      & $(8,8)$      & Tanh & Brier \\
CIFAR-100 & 5{,}000  & $(16,16,16)$ & Tanh & Brier \\
ImageNet  & 500      & $(24,24)$    & ReLU & Brier \\
ImageNet  & 5{,}000  & $(24,24)$    & SiLU & Brier \\
ImageNet  & 25{,}000 & $(16,16,16)$ & ReLU & Brier \\
\bottomrule
\end{tabular}
\end{table}

\begin{table}[h]
\centering
\caption{UMNN: grid-search-selected architecture, fitting loss, and number of integration steps $K$.}
\label{tab:hparam-umnn}
\begin{tabular}{llccc}
\toprule
Dataset & val size & Hidden dims & Loss  \\
\midrule
CIFAR-10  & 500      & $(16,16,16)$ & NLL \\
CIFAR-10  & 5{,}000  & $(8,8)$      & NLL  \\
CIFAR-100 & 500      & $(16,16,16)$ & NLL \\
CIFAR-100 & 5{,}000  & $(16,16,16)$ & NLL  \\
ImageNet  & 500      & $(8,8)$      & NLL  \\
ImageNet  & 5{,}000  & $(16,16)$    & NLL  \\
ImageNet  & 25{,}000 & $(24,24)$    & NLL  \\
\bottomrule
\end{tabular}
\end{table}

\begin{table}[h]
\centering
\caption{Histogram Binning (number of bins) and Spline Calibration (number of knots): grid-search-selected configurations. Only CIFAR-10/100 at val\,5\,000 and ImageNet at all three val sizes are used in the main tables.}
\label{tab:hparam-hb-spline}
\begin{tabular}{llcc}
\toprule
Dataset & val size & Hist.\ Bin.\ (bins) & Spline (knots) \\
\midrule
CIFAR-10  & 5{,}000  & 10 & 9 \\
CIFAR-100 & 5{,}000  & 10 & 9 \\
ImageNet  & 500      & 10 & 9 \\
ImageNet  & 5{,}000  & 10 & 9 \\
ImageNet  & 25{,}000 & 10 & 3 \\
\bottomrule
\end{tabular}

\vspace{0.3em}
{\footnotesize Histogram Binning's ECE is bin-count-invariant here (10/15/20 bins give bit-identical ECE), so the reported ``winner'' is arbitrary among ties.}
\end{table}

\begin{table}[h]
\centering
\caption{Matrix Scaling and Dirichlet Calibration: grid-search-selected regularization strength ($\lambda=\mu$ on the diagonal) and fitting loss. $^\ddagger$On ImageNet, both methods are searched only over the tied diagonal $\lambda=\mu$ (8 configs), not the full $\lambda\times\mu$ cross product used for CIFAR-10/100 (32 configs) -- the full cross product does not fit the compute budget at $C=1{,}000$ classes.}
\label{tab:hparam-ms-dir}
\begin{tabular}{llcccc}
\toprule
 & & \multicolumn{2}{c}{Matrix Sc.} & \multicolumn{2}{c}{Dirichlet} \\
Dataset & val size & $\lambda,\mu$ & Loss & $\lambda,\mu$ & Loss \\
\midrule
CIFAR-10  & 5{,}000       & $0.1,\,100$ & CE & $0.1,\,0.1$ & CE \\
CIFAR-100 & 5{,}000       & $100,\,0.1$ & CE & $10,\,0.1$ & Brier \\
ImageNet$^\ddagger$ & 500      & $10,\,10$ & Brier & $100,\,100$ & Brier \\
ImageNet$^\ddagger$ & 5{,}000  & $10,\,10$ & Brier & $1,\,1$ & Brier \\
ImageNet$^\ddagger$ & 25{,}000 & $100,\,100$ & CE & $1,\,1$ & Brier \\
\bottomrule
\end{tabular}

\vspace{0.3em}
{\footnotesize Matrix Scaling on ImageNet at val\,500/val\,5\,000 selects Brier loss because cross-entropy catastrophically overfits the full $1{,}000\times 1{,}000$ matrix at these val sizes (CE-loss ECE $>20\%$ vs.\ Brier's $\approx\!12.8\%$, see Table~\ref{tab:imagenet_r152}); the resulting Brier-loss fit is statistically indistinguishable from the identity map at those two val sizes.}
\end{table}

\section{Results with Standard Deviations}
\label{app:std}

Tables~\ref{tab:main_v5k} and~\ref{tab:arch} report means over five random seeds for the CIFAR experiments. Table~\ref{tab:std_cifar} provides the corresponding results with standard deviations for all accuracy-preserving methods on CIFAR-10 and CIFAR-100 (ResNet-50, val\,5\,000). InvLT achieves the lowest mean across all metrics on both datasets while maintaining low variance, confirming the stability of the method. Notably, TS on CIFAR-100 exhibits very high variance (ECE std $= 4.08$), driven by a single seed on which the fitted temperature is a poor fit for that seed's logit distribution; ETS/PTS/UMNN/InvLT, which have more flexibility to correct such seeds, do not show the same blowup. Table~\ref{tab:imagenet_r152_std} shows ImageNet's corresponding results with standard deviations.

\newcommand{\val}[2]{#1{\scriptstyle\,\pm\,#2}}

\begin{table*}[htbp]
\centering
\caption{ECE, AdaECE, KDE-ECE (\%$\downarrow$) and NLL ($\downarrow$) on CIFAR-10 and CIFAR-100 (ResNet-50, val\,5\,000, mean $\pm$ std over 5 seeds, accuracy-preserving methods). \textbf{Bold}: best mean.}
\label{tab:std_cifar}
\small
\setlength{\tabcolsep}{3.5pt}
\begin{tabular}{@{}lcccccccc@{}}
\toprule
 & \multicolumn{4}{c}{CIFAR-10} & \multicolumn{4}{c}{CIFAR-100} \\
\cmidrule(lr){2-5}\cmidrule(lr){6-9}
Method & ECE & AdaECE & KDE & NLL & ECE & AdaECE & KDE & NLL \\
\midrule
TS   & $\val{1.69}{.40}$ & $\val{1.54}{.32}$ & $\val{1.86}{.29}$ & $\val{.169}{.003}$
     & $\val{4.67}{4.08}$ & $\val{4.64}{4.09}$ & $\val{5.17}{5.12}$ & $\val{.912}{.042}$ \\
ETS  & $\val{1.64}{.30}$ & $\val{1.48}{.27}$ & $\val{1.81}{.24}$ & $\val{.169}{.003}$
     & $\val{2.72}{.50}$ & $\val{2.78}{.38}$ & $\val{2.81}{.34}$ & $\val{.895}{.012}$ \\
PTS  & $\val{0.69}{.45}$ & $\val{0.71}{.42}$ & $\val{1.15}{.55}$ & $\val{.163}{.006}$
     & $\val{2.59}{.27}$ & $\val{2.85}{.22}$ & $\val{2.79}{.24}$ & $\val{.888}{.015}$ \\
UMNN & $\val{0.66}{.18}$ & $\val{0.74}{.25}$ & $\val{1.10}{.20}$ & $\val{.153}{.003}$
     & $\val{2.48}{.50}$ & $\val{2.60}{.58}$ & $\val{2.61}{.45}$ & $\val{.793}{.012}$ \\
\midrule
\textbf{InvLT} & $\mathbf{\val{0.60}{.03}}$ & $\mathbf{\val{0.61}{.21}}$ & $\mathbf{\val{0.98}{.12}}$ & $\mathbf{\val{.152}{.003}}$
     & $\mathbf{\val{1.67}{.63}}$ & $\mathbf{\val{1.88}{.74}}$ & $\mathbf{\val{1.90}{.59}}$ & $\mathbf{\val{.793}{.012}}$ \\
\bottomrule
\end{tabular}
\end{table*}

\begin{table*}[htbp]
\centering
\caption{ECE, AdaECE, KDE-ECE (\%$\downarrow$) and NLL ($\downarrow$) on ImageNet with ResNet-152 under three calibration set sizes (mean$\pm$std over 5 seeds). The uncalibrated model has ECE $= 12.83\%$, making this a more challenging calibration setting. Methods above the double rule may change the predicted class; below are accuracy-preserving. $^\dagger$KDE-ECE omitted for Dirichlet/Matrix Scaling: re-fitting the grid-search winner to obtain KDE-ECE is seed/hardware-unstable for these $C{=}1{,}000$-class, $\sim\!1$M-parameter fits. $^\ddagger$Matrix Scaling's optimum converges to the identity map at this val size (KDE-ECE therefore equals No Cal.'s exactly). \textbf{Bold}: best. \underline{Underline}: second best.}
\label{tab:imagenet_r152_std}
\scriptsize
\setlength{\tabcolsep}{4pt}

\begin{subtable}{\textwidth}
\centering
\caption{Calibration set size: val\,500}
\label{tab:imagenet_r152_val500}
\begin{tabular}{@{}lcccc@{}}
\toprule
Method & ECE & AdaECE & KDE-ECE & NLL \\
\midrule
\multicolumn{5}{l}{\textit{Accuracy may change}} \\
\midrule
No Cal.     & $\val{12.83}{.19}$ & $\val{12.83}{.19}$ & $\val{12.85}{.19}$ & $\val{.87}{.00}$ \\
Hist.\ Bin. & $\val{5.16}{.60}$  & $\val{6.18}{.23}$  & $\val{5.80}{.40}$  & $\val{3.05}{.03}$ \\
Isotonic    & $\val{33.46}{7.47}$& $\val{33.43}{7.48}$& $\val{34.87}{6.71}$& $\val{12.54}{1.72}$ \\
Vector Sc.  & $\val{43.44}{10.57}$&$\val{43.44}{10.57}$&$\val{43.77}{12.42}$&$\val{23.49}{2.51}$ \\
Matrix Sc.  & $\val{12.83}{.19}^\ddagger$ & $\val{12.83}{.19}$ & $\val{12.85}{.19}$ & $\val{.87}{.00}$ \\
Dirichlet   & $\val{13.20}{1.02}$& $\val{13.33}{.85}$ & --$^\dagger$       & $\val{1.31}{.06}$ \\
Spline      & $\val{5.23}{1.48}$ & $\val{5.19}{1.54}$ & $\val{5.34}{1.39}$ & $\val{1.82}{.05}$ \\
\midrule\midrule
\multicolumn{5}{l}{\textit{Accuracy-preserving}} \\
\midrule
TS    & $\val{4.89}{4.59}$ & $\val{4.88}{4.59}$ & $\val{4.89}{4.59}$ & $\val{.77}{.05}$ \\
ETS   & $\val{2.85}{.23}$  & $\val{2.84}{.19}$  & $\val{2.88}{.14}$  & $\val{.75}{.00}$ \\
PTS   & $\val{3.25}{.59}$  & $\val{3.86}{.33}$  & $\val{3.60}{.40}$  & $\val{.74}{.01}$ \\
UMNN  & $\underline{\val{1.46}{.72}}$ & $\underline{\val{1.46}{.70}}$ & $\underline{\val{1.65}{.76}}$ & $\mathbf{\val{.67}{.01}}$ \\
\midrule
\textbf{InvLT (ours)}
      & $\mathbf{\val{.98}{.43}}$ & $\mathbf{\val{1.00}{.44}}$ & $\mathbf{\val{1.15}{.48}}$ & $\underline{\val{.67}{.01}}$ \\
\bottomrule
\end{tabular}
\end{subtable}

\vspace{1em}

\begin{subtable}{\textwidth}
\centering
\caption{Calibration set size: val\,5\,000}
\label{tab:imagenet_r152_val5k}
\begin{tabular}{@{}lcccc@{}}
\toprule
Method & ECE & AdaECE & KDE-ECE & NLL \\
\midrule
\multicolumn{5}{l}{\textit{Accuracy may change}} \\
\midrule
No Cal.     & $\val{12.83}{.19}$ & $\val{12.83}{.19}$ & $\val{12.85}{.19}$ & $\val{.87}{.00}$ \\
Hist.\ Bin. & $\val{5.61}{.41}$  & $\val{6.93}{.25}$  & $\val{5.87}{.19}$  & $\val{3.49}{.05}$ \\
Isotonic    & $\val{6.84}{.25}$  & $\val{6.84}{.25}$  & $\val{7.62}{.28}$  & $\val{2.97}{.04}$ \\
Vector Sc.  & $\val{10.94}{.31}$ & $\val{10.78}{.35}$ & $\val{11.44}{.33}$ & $\val{2.40}{.08}$ \\
Matrix Sc.  & $\val{12.83}{.19}^\ddagger$ & $\val{12.83}{.19}$ & $\val{12.85}{.19}$ & $\val{.87}{.00}$ \\
Dirichlet   & $\val{3.62}{2.67}$ & $\val{3.64}{2.54}$ & --$^\dagger$       & $\val{1.00}{.07}$ \\
Spline      & $\val{1.10}{.09}$  & $\val{1.19}{.09}$  & $\val{1.34}{.13}$  & $\val{.74}{.00}$ \\
\midrule\midrule
\multicolumn{5}{l}{\textit{Accuracy-preserving}} \\
\midrule
TS    & $\val{2.77}{.21}$ & $\val{2.79}{.18}$ & $\val{2.79}{.13}$ & $\val{.75}{.01}$ \\
ETS   & $\val{2.79}{.23}$ & $\val{2.81}{.18}$ & $\val{2.79}{.14}$ & $\val{.75}{.00}$ \\
PTS   & $\val{2.66}{.14}$ & $\val{3.44}{.26}$ & $\val{3.16}{.22}$ & $\val{.74}{.00}$ \\
UMNN  & $\underline{\val{.56}{.11}}$ & $\underline{\val{.54}{.04}}$ & $\underline{\val{.72}{.14}}$ & $\underline{\val{.67}{.00}}$ \\
\midrule
\textbf{InvLT (ours)}
      & $\mathbf{\val{.39}{.18}}$ & $\mathbf{\val{.50}{.10}}$ & $\mathbf{\val{.64}{.10}}$ & $\mathbf{\val{.67}{.00}}$ \\
\bottomrule
\end{tabular}
\end{subtable}

\vspace{1em}

\begin{subtable}{\textwidth}
\centering
\caption{Calibration set size: val\,25\,000}
\label{tab:imagenet_r152_val25k}
\begin{tabular}{@{}lcccc@{}}
\toprule
Method & ECE & AdaECE & KDE-ECE & NLL \\
\midrule
\multicolumn{5}{l}{\textit{Accuracy may change}} \\
\midrule
No Cal.     & $\val{12.83}{.19}$ & $\val{12.83}{.19}$ & $\val{12.85}{.19}$ & $\val{.87}{.00}$ \\
Hist.\ Bin. & $\val{6.00}{.08}$  & $\val{6.76}{.19}$  & $\val{5.67}{.09}$  & $\val{2.74}{.05}$ \\
Isotonic    & $\val{3.26}{.11}$  & $\val{3.26}{.11}$  & $\val{3.61}{.13}$  & $\val{1.46}{.05}$ \\
Vector Sc.  & $\val{4.74}{.12}$  & $\val{4.73}{.12}$  & $\val{5.07}{.14}$  & $\val{.81}{.01}$ \\
Matrix Sc.  & $\val{10.54}{2.24}$& $\val{10.54}{2.24}$& --$^\dagger$       & $\val{.86}{.01}$ \\
Dirichlet   & $\val{2.66}{.23}$  & $\val{2.80}{.20}$  & --$^\dagger$       & $\val{1.35}{.10}$ \\
Spline      & $\val{1.40}{.13}$  & $\val{1.47}{.08}$  & $\val{1.60}{.15}$  & $\val{.71}{.00}$ \\
\midrule\midrule
\multicolumn{5}{l}{\textit{Accuracy-preserving}} \\
\midrule
TS    & $\val{6.87}{5.60}$ & $\val{6.92}{5.55}$ & $\val{6.89}{5.59}$ & $\val{.80}{.06}$ \\
ETS   & $\val{2.81}{.16}$  & $\val{2.78}{.16}$  & $\val{2.77}{.11}$  & $\val{.75}{.00}$ \\
PTS   & $\val{2.68}{.17}$  & $\val{3.47}{.22}$  & $\val{3.17}{.20}$  & $\val{.74}{.00}$ \\
UMNN  & $\underline{\val{.50}{.18}}$ & $\mathbf{\val{.55}{.16}}$ & $\underline{\val{.71}{.15}}$ & $\underline{\val{.67}{.00}}$ \\
\midrule
\textbf{InvLT (ours)}
      & $\mathbf{\val{.42}{.18}}$ & $\underline{\val{.56}{.08}}$ & $\mathbf{\val{.65}{.12}}$ & $\mathbf{\val{.66}{.00}}$ \\
\bottomrule
\end{tabular}
\end{subtable}

\end{table*}

\section{Additional Ablations}
\label{app:ablation}

Table~\ref{tab:ablation-others} ablates remaining design choices on CIFAR-10 (ResNet-50, val $= 500$). The number of reconstruction samples $K$ controls a trade-off: smaller $K$ achieves slightly better ECE but at the cost of accuracy preservation; $K \ge 25$ ensures accuracy preservation while maintaining competitive calibration. The warmup delay has negligible effect, indicating that the reconstruction loss does not interfere with early optimization. MLP architecture, residual connections, and clipping range are selected per dataset via grid search on the validation set; we find that performance is generally robust across reasonable settings.

\begin{table}[!htbp]
  \centering
  \caption{Ablation on InvLT design choices (CIFAR-10, val $= 500$, ResNet-50). $\dagger$~marks the default used in all other experiments.}
  \label{tab:ablation-others}
  \small
  \setlength{\tabcolsep}{6pt}
  \begin{tabular}{@{}llccc@{}}
    \toprule
    Component & Setting
    & ECE\,(\%$\downarrow$)
    & Error\,(\%$\downarrow$)
    & NLL\,($\downarrow$) \\
    \midrule
    \multirow{5}{*}{Rec.\ samples $K$}
      & 5               & 0.86 & 4.88 & 0.163 \\
      & 10              & 0.92 & 4.72 & 0.157 \\
      & 25              & 1.14 & 4.55 & 0.156 \\
      & $100^{\dagger}$ & 1.17 & 4.53 & 0.156 \\
      & 500             & 1.17 & 4.53 & 0.156 \\
    \midrule
    \multirow{3}{*}{Warmup delay $\tau$}
      & 0                  & 1.17 & 4.54 & 0.156 \\
      & $100^{\dagger}$    & 1.17 & 4.53 & 0.156 \\
      & 1000               & 1.16 & 4.54 & 0.156 \\
    \bottomrule
  \end{tabular}
\end{table}


\newpage
\input{checklist.tex}

\end{document}

%% file: checklist.tex
\clearpage
\section*{NeurIPS Paper Checklist}

\begin{enumerate}

\item {\bf Claims}
    \item[] Question: Do the main claims made in the abstract and introduction accurately reflect the paper's contributions and scope?
    \item[] Answer: \answerYes{} 
    \item[] Justification: The abstract and Section 1 state three concrete contributions—(i) a shared scalar MLP applied element-wise to logits with parameter count independent of C, (ii) a reconstruction-based soft monotonicity regularizer, and (iii) consistent improvements over post-hoc baselines—each of which is substantiated in Sections 3 and 4 (Tables 1–4, Figure 1).

\item {\bf Limitations}
    \item[] Question: Does the paper discuss the limitations of the work performed by the authors?
    \item[] Answer: \answerYes{} 
    \item[] Justification: Limitations are discussed at the end of Section 5: the element-wise design cannot capture cross-class miscalibration dependencies, and the soft monotonicity enforcement does not provide a formal argmax-preservation guarantee. We note that combining InvLT with a post-hoc argmax check may be advisable in safety-critical settings.

\item {\bf Theory assumptions and proofs}
    \item[] Question: For each theoretical result, does the paper provide the full set of assumptions and a complete (and correct) proof?
    \item[] Answer: \answerYes{} 
    \item[] Justification: The two theoretical statements (Theorem 1 on continuous injective functions being strictly monotone, and Corollary 2 on argmax preservation) are stated in Section 3.2 with full proofs given in Appendix A. Assumptions (continuity on an interval, existence of a left inverse) are explicit in the statements.

    \item {\bf Experimental result reproducibility}
    \item[] Question: Does the paper fully disclose all the information needed to reproduce the main experimental results of the paper to the extent that it affects the main claims and/or conclusions of the paper (regardless of whether the code and data are provided or not)?
    \item[] Answer: \answerYes{} 
    \item[] Justification: Section 4.1 specifies datasets (CIFAR-10/100, ImageNet), backbones (ResNet-50/152, VGG-16/19, DenseNet-121, Wide ResNet, ViT-B/16), calibration set sizes, metrics, and baselines. All InvLT hyperparameters are provided in Appendix B (Table 6). 

\item {\bf Open access to data and code}
    \item[] Question: Does the paper provide open access to the data and code, with sufficient instructions to faithfully reproduce the main experimental results, as described in supplemental material?
    \item[] Answer: \answerYes{} 
    \item[] Justification: Anonymized code to reproduce all experiments is provided in the supplementary material

\item {\bf Experimental setting/details}
    \item[] Question: Does the paper specify all the training and test details (e.g., data splits, hyperparameters, how they were chosen, type of optimizer) necessary to understand the results?
    \item[] Answer: \answerYes{} 
    \item[] Justification: Section 4.1 describes data splits (val 500 / 5 000 / 25 000), evaluation metrics (ECE, AdaECE, KDE-ECE, NLL), number of seeds (5), and baselines. Optimizer (Adam) and full hyperparameter grid are documented in Section 3.3 and Appendix B (Table 6).

\item {\bf Experiment statistical significance}
    \item[] Question: Does the paper report error bars suitably and correctly defined or other appropriate information about the statistical significance of the experiments?
    \item[] Answer: \answerYes{} 
    \item[] Justification: All main results are reported as means over 5 random seeds. Standard deviations across seeds are provided in Appendix C (Tables 7 and 8) for both CIFAR and ImageNet experiments. Variability captures the random seed used for calibrator fitting.

\item {\bf Experiments compute resources}
    \item[] Question: For each experiment, does the paper provide sufficient information on the computer resources (type of compute workers, memory, time of execution) needed to reproduce the experiments?
    \item[] Answer: \answerYes{} 
    \item[] Justification: Wall-clock fitting and inference times are reported in Table 4 on a single CPU. Calibration is lightweight (fitting completes in seconds to tens of seconds), 
    
\item {\bf Code of ethics}
    \item[] Question: Does the research conducted in the paper conform, in every respect, with the NeurIPS Code of Ethics \url{https://neurips.cc/public/EthicsGuidelines}?
    \item[] Answer: \answerYes{} 
    \item[] Justification: The research uses only public benchmarks (CIFAR-10/100, ImageNet) and does not involve human subjects, sensitive personal data, or deployments in high-risk settings. We are not aware of any deviation from the NeurIPS Code of Ethics.

\item {\bf Broader impacts}
    \item[] Question: Does the paper discuss both potential positive societal impacts and negative societal impacts of the work performed?
    \item[] Answer: \answerYes{} 
    \item[] Justification: Better-calibrated classifiers support more trustworthy decision-making in downstream applications (medical diagnosis, autonomous systems, selective prediction), as motivated in Section 1. Potential negative impact is limited: post-hoc calibration does not change classifier predictions and could give a false sense of reliability if calibration sets are not representative of deployment data; users should validate calibration on representative data before relying on calibrated probabilities.
    
\item {\bf Safeguards}
    \item[] Question: Does the paper describe safeguards that have been put in place for responsible release of data or models that have a high risk for misuse (e.g., pre-trained language models, image generators, or scraped datasets)?
    \item[] Answer: \answerNA{} 
    \item[] Justification: The paper does not release pretrained generative models, scraped datasets, or other assets with elevated misuse risk. The proposed calibrator is a small MLP that operates on the logits of an already-trained classifier.

\item {\bf Licenses for existing assets}
    \item[] Question: Are the creators or original owners of assets (e.g., code, data, models), used in the paper, properly credited and are the license and terms of use explicitly mentioned and properly respected?
    \item[] Answer: \answerYes{} 
    \item[] Justification: All datasets (CIFAR-10/100, ImageNet) and architectures (ResNet, VGG, DenseNet, Wide ResNet, ViT) are cited in Section 4.1. Baseline methods (TS, ETS, PTS, UMNN, Spline, Dirichlet, Histogram Binning, Isotonic, Platt, Vector/Matrix scaling) are cited at first use. License terms of these public assets are respected. 

\item {\bf New assets}
    \item[] Question: Are new assets introduced in the paper well documented and is the documentation provided alongside the assets?
    \item[] Answer: \answerYes{} 
    \item[] Justification: we will provided a documented code base along with this paper

\item {\bf Crowdsourcing and research with human subjects}
    \item[] Question: For crowdsourcing experiments and research with human subjects, does the paper include the full text of instructions given to participants and screenshots, if applicable, as well as details about compensation (if any)? 
    \item[] Answer: \answerNA{} 
    \item[] Justification: The work involves no crowdsourcing or human-subject experiments.

\item {\bf Institutional review board (IRB) approvals or equivalent for research with human subjects}
    \item[] Question: Does the paper describe potential risks incurred by study participants, whether such risks were disclosed to the subjects, and whether Institutional Review Board (IRB) approvals (or an equivalent approval/review based on the requirements of your country or institution) were obtained?
    \item[] Answer: \answerNA{}{} 
    \item[] Justification: The work involves no human-subject research and therefore does not require IRB review.

\item {\bf Declaration of LLM usage}
    \item[] Question: Does the paper describe the usage of LLMs if it is an important, original, or non-standard component of the core methods in this research? Note that if the LLM is used only for writing, editing, or formatting purposes and does \emph{not} impact the core methodology, scientific rigor, or originality of the research, declaration is not required.
    \item[] Answer: \answerNA{} 
    \item[] Justification: LLMs were not used as a component of the core methodology. Any use of LLMs was limited to non-substantive editing of writing, which does not require declaration per the NeurIPS policy.

\end{enumerate}

%% file: references.bib
@inproceedings{guo2017calibration,
  title={On calibration of modern neural networks},
  author={Guo, Chuan and Pleiss, Geoff and Sun, Yu and Weinberger, Kilian Q},
  booktitle={International conference on machine learning},
  pages={1321--1330},
  year={2017},
  organization={PMLR}
}

@article{platt1999probabilistic,
  title={Probabilistic outputs for support vector machines and comparisons to regularized likelihood methods},
  author={Platt, John and others},
  journal={Advances in large margin classifiers},
  volume={10},
  number={3},
  pages={61--74},
  year={1999},
  publisher={Cambridge, MA}
}

@inproceedings{naeini2015obtaining,
  title={Obtaining well calibrated probabilities using bayesian binning},
  author={Naeini, Mahdi Pakdaman and Cooper, Gregory and Hauskrecht, Milos},
  booktitle={Proceedings of the AAAI conference on artificial intelligence},
  volume={29},
  number={1},
  year={2015}
}

@inproceedings{zhang2020mix,
  title={Mix-n-match: Ensemble and compositional methods for uncertainty calibration in deep learning},
  author={Zhang, Jize and Kailkhura, Bhavya and Han, T Yong-Jin},
  booktitle={International conference on machine learning},
  pages={11117--11128},
  year={2020},
  organization={PMLR}
}

@article{rahimi2020intra,
  title={Intra order-preserving functions for calibration of multi-class neural networks},
  author={Rahimi, Amir and Shaban, Amirreza and Cheng, Ching-An and Hartley, Richard and Boots, Byron},
  journal={Advances in neural information processing systems},
  volume={33},
  pages={13456--13467},
  year={2020}
}

@inproceedings{tomani2022parameterized,
  title={Parameterized temperature scaling for boosting the expressive power in post-hoc uncertainty calibration},
  author={Tomani, Christian and Cremers, Daniel and Buettner, Florian},
  booktitle={European conference on computer vision},
  pages={555--569},
  year={2022},
  organization={Springer}
}

@inproceedings{zadrozny2001obtaining,
  title={Obtaining calibrated probability estimates from decision trees and naive bayesian classifiers},
  author={Zadrozny, Bianca and Elkan, Charles},
  booktitle={Icml},
  volume={1},
  number={05},
  pages={2001},
  year={2001}
}

@inproceedings{zadrozny2002transforming,
  title={Transforming classifier scores into accurate multiclass probability estimates},
  author={Zadrozny, Bianca and Elkan, Charles},
  booktitle={Proceedings of the eighth ACM SIGKDD international conference on Knowledge discovery and data mining},
  pages={694--699},
  year={2002}
}

@article{gupta2020spline,
  title={Calibration of neural networks using splines},
  author={Gupta, Kartik and Rahimi, Amir and Ajanthan, Thalaiyasingam and Mensink, Thomas and Sminchisescu, Cristian and Hartley, Richard},
  journal={arXiv preprint arXiv:2006.12800},
  year={2020}
}

@article{kull2019dirichlet,
  title={Beyond temperature scaling: Obtaining well-calibrated multi-class probabilities with dirichlet calibration},
  author={Kull, Meelis and Perello Nieto, Miquel and K{\"a}ngsepp, Markus and Silva Filho, Telmo and Song, Hao and Flach, Peter},
  journal={Advances in neural information processing systems},
  volume={32},
  year={2019}
}

@article{wehenkel2019umnn,
  title={Unconstrained monotonic neural networks},
  author={Wehenkel, Antoine and Louppe, Gilles},
  journal={Advances in neural information processing systems},
  volume={32},
  year={2019}
}

@inproceedings{nixon2019measuring,
  title={Measuring calibration in deep learning.},
  author={Nixon, Jeremy and Dusenberry, Michael W and Zhang, Linchuan and Jerfel, Ghassen and Tran, Dustin},
  booktitle={CVPR workshops},
  number={7},
  year={2019}
}

@article{minderer2021revisiting,
  title={Revisiting the calibration of modern neural networks},
  author={Minderer, Matthias and Djolonga, Josip and Romijnders, Rob and Hubis, Frances and Zhai, Xiaohua and Houlsby, Neil and Tran, Dustin and Lucic, Mario},
  journal={Advances in neural information processing systems},
  volume={34},
  pages={15682--15694},
  year={2021}
}

@inproceedings{zhu2017unpaired,
  title={Unpaired image-to-image translation using cycle-consistent adversarial networks},
  author={Zhu, Jun-Yan and Park, Taesung and Isola, Phillip and Efros, Alexei A},
  booktitle={Proceedings of the IEEE international conference on computer vision},
  pages={2223--2232},
  year={2017}
}

@article{kingma2014adam,
  title={Adam: A method for stochastic optimization},
  author={Kingma, Diederik P and Ba, Jimmy},
  journal={arXiv preprint arXiv:1412.6980},
  year={2014}
}

@article{krizhevsky2009learning,
  title={Learning multiple layers of features from tiny images},
  author={Krizhevsky, Alex and Hinton, Geoffrey and others},
  year={2009},
  publisher={Toronto, ON, Canada}
}

@inproceedings{deng2009imagenet,
  title={Imagenet: A large-scale hierarchical image database},
  author={Deng, Jia and Dong, Wei and Socher, Richard and Li, Li-Jia and Li, Kai and Fei-Fei, Li},
  booktitle={2009 IEEE conference on computer vision and pattern recognition},
  pages={248--255},
  year={2009},
  organization={Ieee}
}

@article{simonyan2014vgg,
  title={Very deep convolutional networks for large-scale image recognition},
  author={Simonyan, Karen and Zisserman, Andrew},
  journal={arXiv preprint arXiv:1409.1556},
  year={2014}
}

@inproceedings{huang2017densely,
  title={Densely connected convolutional networks},
  author={Huang, Gao and Liu, Zhuang and Van Der Maaten, Laurens and Weinberger, Kilian Q},
  booktitle={Proceedings of the IEEE conference on computer vision and pattern recognition},
  pages={4700--4708},
  year={2017}
}

@article{zagoruyko2016wide,
  title={Wide residual networks},
  author={Zagoruyko, Sergey and Komodakis, Nikos},
  journal={arXiv preprint arXiv:1605.07146},
  year={2016}
}

@inproceedings{dosovitskiy2021vit,
  title={An Image is Worth 16x16 Words: Transformers for Image Recognition at Scale},
  author={Dosovitskiy, Alexey and Beyer, Lucas and Kolesnikov, Alexander and Weissenborn, Dirk and Zhai, Xiaohua and Unterthiner, Thomas and Dehghani, Mostafa and Minderer, Matthias and Heigold, Georg and Gelly, Sylvain and Uszkoreit, Jakob and Houlsby, Neil},
  booktitle={Proceedings of the 9th International Conference on Learning Representations (ICLR)},
  year={2021}
}

@inproceedings{he2016resnet,
  title={Deep residual learning for image recognition},
  author={He, Kaiming and Zhang, Xiangyu and Ren, Shaoqing and Sun, Jian},
  booktitle={Proceedings of the IEEE conference on computer vision and pattern recognition},
  pages={770--778},
  year={2016}
}
